\documentclass[review,11pt]{ReportTemplate}
\usepackage{bm}
\usepackage{graphicx}
\usepackage{paralist,algorithmic,algorithm}
\usepackage{amsmath,mathrsfs,amssymb}
\newtheorem{thm}{Theorem}

\newtheorem{definition}{Definition}

\newtheorem{lemma}{Lemma}

\def \x {\mathbf{x}}
\def \z {\mathbf{z}}
\def \w {\mathbf{w}}

\def \X {\mathcal{X}}
\def \Y {\mathcal{Y}}
\def \Z {\mathcal{Z}}
\def \R {\mathcal{R}}
\def \W {\mathcal{W}}
\def \F {\mathcal{F}}
\def \H {\mathcal{H}}
\def \R {\mathcal{R}}
\def \r {\mathbf{r}}
\newenvironment{proof}{\textbf{Proof:}\ }{\hspace{\stretch{1}}$\square$\\}

\begin{document}
\begin{frontmatter}
\title{Dropout Rademacher Complexity of Deep Neural Networks}

\author{Wei Gao}
\author{Zhi-Hua Zhou\corref{cor1}}
\address{National Key Laboratory for Novel Software Technology\\
Nanjing University, Nanjing 210093, China} \cortext[cor1]{\small Corresponding author.
Email: zhouzh@nju.edu.cn}

\begin{abstract}
Great successes of deep neural networks have been witnessed in various real applications. Many algorithmic and implementation techniques have been developed; however, theoretical understanding of many aspects of deep neural networks is far from clear. A particular interesting issue is the usefulness of \textit{dropout}, which was motivated from the intuition of  preventing complex co-adaptation of feature detectors. In this paper, we study the Rademacher complexity of different types of dropout, and our theoretical results disclose that for shallow neural networks (with one or none hidden layer) dropout is able to reduce the Rademacher complexity in polynomial, whereas for deep neural networks it can amazingly lead to an exponential reduction.
\end{abstract}

\end{frontmatter}

Deep neural networks \cite{Hinton:Salakhutdinov2006} has become a hot wave during the past few years, and great successes have been achieved in various applications, such as object recognition \cite{Ciresan:Meier:Gambardella:Schmidhuber2010,Ciresan:Meier:Schmidhuber2012,Coates:Huval:Wang:Wu:Ng:Catanzaro2013}, video analysis \cite{Mobahi:Collobert:Weston2009,Goodfellow:Lee:Le:Saxe:Ng2009,Bo:Lai:Ren:Fox2011}, speech recognition \cite{Dahl:Yu:Deng:Acero2012,Hinton:Deng:Yu2012,Dahl:Sainath:Hinton2013}, etc. Many effective algorithmic and implementation techniques have been developed; however, theoretical understanding of many aspects of deep neural networks is far from clear.

It is well known that deep neural networks are complicated models with rich representations. For really deep networks, there may be millions or even billions of parameters, and thus, there are high risks of overfitting even with large-scale training data. Indeed, controlling the overfitting risk is a long-standing topic in the research of neural networks, and various techniques have been developed, such as weight elimination \cite{Andreas:David:Bernardo1991}, early stopping \cite{Amari:Murata:Muller:Finke:Yang1997}, Bayesian control \cite{Neal1996}, etc.

Dropout is among the key ingredients of the success of deep neural networks. The main idea is to randomly omit some units, either hidden ones or input ones corresponding to different input features; this is executed with certain probability in the forward propagation of training phase, and the weights related to the remaining units are updated in back propagation. This technique is evidently related to overfitting control, though it was proposed with the intuition of preventing complex co-adaptations by encouraging independent contributions from different features during training phase \cite{Hinton:Srivastava:Krizhevsky:Sutskever:Salakhutdinov2012}. Extensive empirical studies \cite{Hinton:Srivastava:Krizhevsky:Sutskever:Salakhutdinov2012,Krizhevsky:Sutskever:Hinton2012,Wan:Zeiler:Zhang:Cun:Fergus2013} verified that dropout is able to improve the performance and reduce ovefitting risk. However, theoretical understanding of dropout is far from clear.

In this paper, we study the influence on Rademacher complexity by three types of dropouts, i.e., dropout of units \cite{Hinton:Srivastava:Krizhevsky:Sutskever:Salakhutdinov2012}, dropout of weights \cite{Wan:Zeiler:Zhang:Cun:Fergus2013} and dropout both. Our theoretical results disclose that for shallow neural networks with none or one hidden layer, dropout is able to reduce the Rademacher complexity in polynomial, whereas for deep neural networks it is able to reach an exponential reduction of Rademacher complexity.

\subsection*{Related Work}
There are several designs of dropout, such as the fast dropout \cite{Wang:Manning2013} and adaptive dropout \cite{Ba:Frey2013}, whereas the most fundamental dropouts are the dropout of units (hidden units, or input units corresponding to input features) \cite{Hinton:Srivastava:Krizhevsky:Sutskever:Salakhutdinov2012}  and the dropout of weights \cite{Wan:Zeiler:Zhang:Cun:Fergus2013}.

Baldi and sadowski [2013] studied the average and regularizing properties of dropout, and Wager et al. [2013] showed that dropout is first-order regular equivalent to an $L$ regularizer applied after scaling the features by an estimate of the inverse diagonal Fisher information matrix. The generalization bound of dropout has been analyzed in \cite{McAllester2013,Wan:Zeiler:Zhang:Cun:Fergus2013}. McAllester [2013] presented PAC-Bayesian bounds, whereas Wan et al. [2013] derived Rademacher generalization bounds. Both their results show that the reduction of complexity brought by dropout is  O($\rho$), where $\rho$ is the probability of keeping an element in dropout.

In contrast to previous studies \cite{McAllester2013,Wan:Zeiler:Zhang:Cun:Fergus2013}, we present better generalization bounds and disclose that dropout is able to reduce the Rademacher complexity exponentially, i.e., O($\rho^{k+1}$) or O($\rho^{(k+1)/2}$) for different types of dropout, where $k$ is the number of hidden layers within neural networks.

Extensive work \cite[and reference therein]{Karpinski:Macintyre1997,Anthony:Bartlett2009} studied the complexity of neural network based on VC-dimension, covering number, fat-shatter dimension, etc., and it was usually shown that these complexities are polynomial in the total number of units and weights. Note that the polynomial complexities are still very high for deep neural networks that may have million or even billions of parameters. Moreover, it is worth noting that these complexities measure the function space in the worst case, and cannot distinguish situations with/without dropouts. In contrast, we show that Radermacher complexity is proper to study the influence of dropouts, and we prove that the complexities of neural network can be bounded by the $L_1$ or $L_2$-norm of weights, irrelevant to the number of units and weights.

This paper is organized as follows: Section~\ref{sec:pre} introduces some preliminaries. Section~\ref{sec:bound} presents general Rademacher generalization bounds for dropout. Section~\ref{sec:rademacher} analyzes the usefulness of different types of dropouts on shallow as well as deep neural networks. Section~\ref{sec:con} concludes.

\section{Preliminaries}\label{sec:pre}
Let $\X\subset\mathbb{R}^d$ and $\mathcal{Y}$ be the input and output space, respectively, where $\Y\subset \mathbb{R}$ for regression and $\Y=\{+1,-1\}$ for binary classification. Throughout this paper, we restrict our attention on regression and binary classification, and it is easy to make similar analysis for multi-class tasks. Let $\mathcal{D}$ be an unknown (underlying) distribution over $\X\times\Y$.

Let $\W$ be the weight space for neural network, and denote by  $f(\w,\x)$ the general output of a neural network with respect to input $\x\in\X$ and weight $\w\in\W$. Here $f$ depends on the structure of neural network. During training neural network, dropout randomly omits hidden units, input units corresponding to input features, and connected weights with certain probability; therefore, it is necessary to introduce another space
\[
\R=\{\r=(r_1,r_2,\ldots,r_s)\colon r_i\in\{0,1\}\}
\]
where $s$ depends on different neural networks and different types of dropout, and $r_i=0$ implies dropping out some hidden unit, input unit and weight. Here each $r_i$ is drawn independently and identically from a Bernoulli distribution with parameter $\rho$, denoted by Bern$(\rho)$. Further, we denote by $f(\w,\x,\r)$ the dropout output of a neural network, and write
\[
\F_\W=\{f(\w,\x,\r)\colon \w\in\W\},
\]
as the function space for dropout. Here we just present a general output $f(\w,\x,\r)$ for dropout, and detailed expressions will be given for specific neural network in Section~\ref{sec:rademacher}.

An objective function (or loss function) $\ell$ is introduced to measure the performance of output of neural network. For example, least square loss and cross entropy are used for regression and binary classification, respectively. We define the expected risk for dropout as
\[
R(\w)=E_{\r,(\x,y)}[\ell(f(\w,\x,\r),y)].
\]
The goal is to find a $\w^*\in\W$ so as to minimize the expected risk, i.e., $\w^*\in\arg\min_{\w\in\W} R(\w)$. Notice that the distribution $\mathcal{D}$ is unknown, but it is demonstrated by a training sample
\[
S_n=\{(\x_1,y_1),(\x_2,y_2), \ldots, (\x_n,y_n)\}
\]
which are drawn i.i.d. from distribution $\mathcal{D}$. Given sample $S_n$ and $RS_n=\{\r_1,\r_2,\ldots,\r_n\}$, we define the empirical risk for dropout as
\[
\hat{R}(\w,S_n,RS_n)=\frac{1}{n}\sum_{i=1}^n \ell(f(\w,\x_i,\r_i),y_i).
\]
In this paper, we try to study on generalization bounds for dropouts, i.e., the gap between $R(\w)$ and $\hat{R}(\w,S_n,RS_n)$. Rademacher complexity has always been an efficient measure for function space \cite{Bartlett:Mendelson2002,Koltchinskii:Panchenko2002}. For function space $\H$, the classical Rademacher complexity is defined by
\begin{equation}\label{eq:00}
\hat{\mathfrak{R}}_n(\H)=E\Big[\sup_{h\in\H}\frac{1}{n}\sum_{i=1}^n \epsilon_i h(\x_i) \Big]
\end{equation}
where $\epsilon_1, \ldots, \epsilon_n$ are independent random variables uniformly chosen from $\{+1,-1\}$, and they are referred as Rademacher variables. Rademacher complexity has been used to develop data-dependent generalization bounds in diverse learning tasks \cite{Meir:Zhang2003,Maurer2006,Corte:Mohri:Rostamizadeh2010}.

For notational simplicity, we denote by $[n]=\{1,2,\ldots,n\}$ for integer $n>0$. The inner product between $\w=(w_1,\ldots,w_d)$ and $\x=(x_1,\ldots,x_d)$ is given by $\langle\w,\x\rangle=\sum\nolimits_{i=1}^d w_ix_i$, and write $\|\w\|=\|\w\|_2=\sqrt{\langle\w,\w\rangle}$ and $\|\w\|_1=\sum_{i=1}^d|w_i|$. Further, the entrywise product (also called Schur product or Hadamard product) is defined as $\w\odot\x=(w_1x_1,\ldots,w_dx_d)$.

\section{General Rademacher Generalization Bounds}\label{sec:bound}
In conventional studies, the generalization performance is mostly affected by training sample, and thus, standard Rademacher complexity is defined on training sample only (as shown in Eq.~\ref{eq:00}). For dropout, however, the generalization performance is not only relevant to training sample, but also dropout randomization; thus, we generalize the Rademacher complexity as follows:
\begin{definition}
For spaces $\Z$ and $\R$, let $\H\colon\Z\times\R\to\mathbb{R}$ be a real-valued function space. For  $S_n=\{\z_1,\ldots,\z_n\}$ and $RS_n=\{\r_1,\ldots,\r_n\}$, the empirical Rademacher complexity of $\H$ is defined to be
\[
\hat{\mathfrak{R}}_n(\H,S_n,RS_n)= E_{\epsilon}\Big[\sup_{h\in\H} \Big(\frac{1}{n}\sum_{i=1}^n \epsilon_i h(\z_i,\r_i)\Big) \Big]
\]
where $\epsilon=(\epsilon_1,\ldots,\epsilon_n)$ are Rademacher variables. Further, we define the Rademacher complexity of $\H$ as
\[
\mathfrak{R}_n(\H)= E_{S_n,RS_n}[\hat{\mathfrak{R}}_n(\H,S_n,RS_n)].
\]
\end{definition}

Based on this definition, it is easy to get a useful lemma as follows:
\begin{lemma}\label{lem:convex}
For function space $\H$, define $\text{absconv}(\H)=\{\sum \alpha_i h_i\colon h_i\in\H \text{ and } \sum|\alpha_i|=1\}$. Then, we have
\[
\hat{\mathfrak{R}}_n(\H,S_n,RS_n)=\hat{\mathfrak{R}}_n(\text{absconv}(\H),S_n,RS_n).
\]
\end{lemma}

Given a set $\W$, we denote by composite function space for dropout as
\[
\ell\circ\F_\W:=\{((\x,y),\r)\to\ell(f(\w,\x,\r),y), \w\in\W\}.
\]

Based on the generalized Rademacher complexity, we present the general Rademacher generalization bounds for dropout as follows:
\begin{thm}\label{thm:main}
Let $S_n=\{(\x_1,y_1),(\x_2,y_2), \ldots, (\x_n,y_n)\}$ be a sample chosen i.i.d. according to distribution $\mathcal{D}$, and let $RS_n=\{\r_1,\r_2,\ldots,\r_n\}$ be random variable sample for dropout. If the loss function $\ell$ is bounded by $B>0$, then for every $\delta>0$ and $\w\in\mathcal{W}$, the following holds with probability at least $1-\delta$,
\begin{eqnarray}
&R(\w)\leq \hat{R}(\w,S_n,RS_n)+ 2 \mathfrak{R}_n(\ell\circ\F_\W) +B\sqrt{\ln(2/\delta)/n},&\label{eq:main2}\\
&R(\w)\leq \hat{R}(\w,S_n,RS_n)+ 2 \hat{\mathfrak{R}}_n(\ell\circ\F_\W,S_n,RS_n) +3B\sqrt{\ln(2/\delta)/n}.&\label{eq:main1}
\end{eqnarray}
\end{thm}

The proof is motivated from the techniques in \cite{Bartlett:Mendelson2002}. We can easily find that the difference between Eq.~\ref{eq:main2} and the bound without dropout from \cite{Koltchinskii:Panchenko2002} is a constant $1/\sqrt{2}$.

\begin{proof}
For every $\w\in\mathcal{W}$, it is easy to observe
\[
R(\w)\leq \hat{R}(\w,S_n,RS_n) + \sup_\w[R(\w)-\hat{R}(\w,S_n,RS_n)],
\]
and we further denote by
\[
\Phi(S_n,RS_n)=\sup_\w[R(\w)-\hat{R}(\w,S_n,RS_n)]=\sup_\w\Big[R(\w)-\frac{1}{n}\sum_{i=1}^n \ell(f(\w,\x_i,\r_i),y_i)\Big].
\]
Let $S_n^{i,(\x'_i,y'_i)}=\{(\x_1,y_1),\ldots,(\x'_i,y'_i),\ldots,(\x_n,y_n)\}$ be the sample whose $i$-th example $(\x_i,y_i)$ in $S_n$ is replaced by $(\x'_i,y'_i)$, and $RS_n^{i,\r'_i}=\{\r_1,\ldots,\r'_i,\ldots,\r_n\}$ be the random variable vector with $i$-th variable $\r_i$ replaced by $\r'_i$. For bounded loss $|\ell|<B$, we have
\[
|\Phi(S_n,RS_n)-\Phi(S_n,RS_n^{i,\r'_i})|\leq B/m \text{ and } |\Phi(S_n,RS_n)-\Phi(S_n^{i,(\x'_i,y'_i)},RS_n)|\leq B/m.
\]
Based on McDiarmid's inequality \cite{McDiarmid1989}, it holds that with probability at least $1-\delta$,
\[
\Phi(S_n,RS_n)\leq E_{S_n,RS_n}[\Phi(S_n,RS_n)]+B\sqrt{\ln(2/\delta)/n}.
\]
Define a ghost sample $\tilde{S}_n=\{(\tilde{\x}_1,\tilde{y}_1),\ldots, (\tilde{\x}_n,\tilde{y}_n)\}$ and a ghost random variable vector $\widetilde{RS}_n=\{\tilde{\r}_1,\ldots,\tilde{\r}_1\}$. By using the fact
\[
\Phi(S_n,RS_n)=\sup_\w[E_{\tilde{S}_n,\widetilde{RS}_n} [\hat{R}(\w,\tilde{S}_n,\widetilde{RS}_n)-\hat{R}(\w,S_n,RS_n)]],
\]
we have
\begin{eqnarray*}
E_{S_n,RS_n}[\Phi(S_n,RS_n)]&\leq& E\left[\sup_{\w} \left[\hat{R}(\w,\tilde{S}_n,\widetilde{RS}_n)- \hat{R}(\w,S_n,RS_n)\right]\right]\\
&=& E\left[\sup_{\w} \left[\frac{\sum_{i=1}^n \ell(f(\w,\tilde{\x}_i, \tilde{\r}_i), \tilde{y}_i)- \ell(f(\w,\x_i,\r_i),y_i)}{n}\right]\right]\\
&\leq& 2 E\left[\sup_{\w}\frac{1}{n}\sum_{i=1}^n \epsilon_i \ell(f(\w,\x_i,\r_i),y_i) \right]=2\mathfrak{R}_n(\ell\circ\F_\W)
\end{eqnarray*}
which completes the proofs of Eq.~\ref{eq:main2}. Again, we apply McDiarmid's inequality to $\hat{\mathfrak{R}}_n(\mathcal{W},S_n,RS_n)$, we have
\[
\mathfrak{R}_n(\ell\circ\F_\W) \leq \hat{\mathfrak{R}}_n(\ell\circ\F_\W,S_n,RS_n) + B\sqrt{\ln(2/\delta)/n}
\]
which completes the proof of Eq.~\ref{eq:main1}.
\end{proof}

The main benefit of dropout lies in the sharp reduction on Rademacher complexities of $\mathfrak{R}_n(\F_\W)$ as will been shown in Section~\ref{sec:rademacher}; on the other hand, extensive experiments show that dropout decreases the empirical risk $\hat{R}(\w,S_n,RS_n)$ \cite{Hinton:Srivastava:Krizhevsky:Sutskever:Salakhutdinov2012,Krizhevsky:Sutskever:Hinton2012,Wan:Zeiler:Zhang:Cun:Fergus2013} because dropout intuitively prevents complex co-adaptations by encouraging independent contributions from different features during training phase \cite{Hinton:Srivastava:Krizhevsky:Sutskever:Salakhutdinov2012}. This paper tries to present theoretical analysis on the the former, and leave the latter to future work.

To efficiently estimate $\mathfrak{R}_n(\ell\circ\F_\W)$, we introduce a concentration as follows:
\begin{lemma}\cite{Ledoux:Talagrand2002}\label{lem:concentration}
Let $\mathcal{H}$ be a bounded real-valued function space from some space $\Z$ and $\z_1,\ldots,\z_n\in\Z$. Let $\phi\colon\mathbb{R}\to\mathbb{R}$ be Lipschitz with constant $L$ and $\phi(0)=0$. Then, we have
\[
E_{\epsilon}\sup_{h\in\H}\frac{1}{n}\sum_{i\in[n]}\epsilon_i\phi(h(\z_i))\leq L E_{\epsilon}\sup_{h\in\H}\frac{1}{n}\sum_{i\in[n]}\epsilon_ih(\z_i).
\]
\end{lemma}
Based on this lemma, we have
\begin{lemma}\label{lem:temp}
If $\ell(\cdot,\cdot)$ is Lipschitz with the first argument and constant $L$, then we have
\[
\mathfrak{R}_n(\ell\circ\F_\W)\leq L \mathfrak{R}_n(\F_\W).
\]
\end{lemma}
\begin{proof}
We first write $\ell'(\cdot,\cdot)=\ell(\cdot,\cdot)-\ell(0,\cdot)$, and it is easy to get $\mathfrak{R}_n(\ell\circ\F_\W)=\mathfrak{R}_n(\ell'\circ\F_\W)$.
This lemma holds by applying Lemma~\ref{lem:concentration} to $\ell'$.
\end{proof}

For classification, we always use the entropy loss as the loss function in neural network as follows:
\[
\ell(f(\w,\x,\r),y)=y\ln(y/\phi(f(\w,\x,\r)))+(1-y)\ln((1-y)/(1-\phi(f(\w,\x,\r)))),
\]
where $\phi(t)=1/(1+e^{-t})$. It is easy to find that $\partial\ell(f(\w,\x,\r),y)/\partial f(\w,\x,\r)\in[-1,1]$, and thus $\ell(\cdot,\cdot)$ is a Lipschitz function with the first argument.

For regression, we always use the square loss as the loss function in neural network as follows:
\[
\ell(f(\w,\x,\r),y)=(y-f(\w,\x,\r))^2.
\]
For bounded $f(\w,\x,\r)$ and $y$, it is easy to find that $\ell(\cdot,\cdot)$ is a Lipschitz function with the first argument.

Based on Lemma~\ref{lem:temp}, it is easy to estimate $\mathfrak{R}_n( \ell \circ \F_\W)$ from $\mathfrak{R}_n(\F_\W)$; therefore, we will focus on how to estimate $\mathfrak{R}_n(\F_\W)$ for different types of dropouts and different neural networks in the subsequent section.

Finally, we introduce a useful lemma as follows:
\begin{lemma}
Let $\r_1=(r_{11},\ldots,r_{1d})$ and $\r_2=(r_{21},\ldots,$ $r_{2d})$  be two random variable vectors, and each element in $\r_1$ and $\r_2$ is drawn i.i.d. from distribution Bern$(\rho)$. For $\x\in\X$, we have
\begin{eqnarray}
&E_{\r_1}[\langle \x\odot \r_1,\x\odot \r_1 \rangle]=\rho\langle\x,\x\rangle, \label{eq:ulem2}\\
&E_{\r_1,\r_2}[\langle \x\odot \r_1\odot \r_2,\x\odot \r_1 \odot \r_2 \rangle]=\rho^2\langle \x,\x\rangle \label{eq:ulem3}.
\end{eqnarray}
Further, let $\r=(r_1,\ldots,r_k)$ be $k$ random variables drawn i.i.d. from  distribution Bern$(\rho)$. We have
\begin{eqnarray}
&\mathop{E}_{\r_1,\r}\Big\langle \x\odot \r_1\prod_{i=1}^k r_i, \x \odot \r_1\prod_{i=1}^k r_i\Big\rangle=\rho^{k+1}\langle\x,\x\rangle,&\label{eq:ulem4}\\
&\mathop{E}_{\r, \r_1,\r_2}\Big\langle \x\odot \r_1\odot \r_2 \prod_{i=1}^k r_i,
 \x \odot \r_1\odot \r_2 \prod_{i=1}^k r_i  \Big\rangle
=\rho^{k+2}\langle\x,\x\rangle.&\label{eq:ulem5}
\end{eqnarray}
\end{lemma}
\begin{proof}
Let $\x=(x_1,\ldots,x_d)$. From the definitions of inner product and entrywise product, Eq.~\ref{eq:ulem2} holds from
\[
E_{\r_1}[\langle \x\odot \r_1,\x\odot \r_1 \rangle]= E_{\r_1}\Big[\sum_{j=1}^d x_{j}x_{j}r^2_{1j}\Big]=\rho\sum_{j=1}^d x_{j}x_{j}
\]
where we use the fact $E_{r_{1j}}[r^2_{1j}]=\rho$ since $r_{1j}$ is drawn i.i.d. from distribution $\text{Bern}(\rho)$. In a similar manner, Eqs.~\ref{eq:ulem3}-\ref{eq:ulem5} hold from $E_{r_{1j},r_{2j}} [r^2_{1j}r^2_{2j}]=\rho^2$, $E_{\r,r_{1j}}[r_{1j}\prod_{i=1}^k (r_i)^2] =\rho^{1+k}$ and $E_{\r,r_{1j},r_{2j}}[r^2_{1j}r^2_{2j}\prod_{i=1}^k (r_i)^2]=\rho^{2+k}$, respectively. This lemma follows as desired.
\end{proof}

\section{Dropouts on Different Types of Network}\label{sec:rademacher}
We study the two types of most fundamental dropouts: dropout units \cite{Hinton:Srivastava:Krizhevsky:Sutskever:Salakhutdinov2012} and dropout weights \cite{Wan:Zeiler:Zhang:Cun:Fergus2013}. In addition, we also study dropout both units and weights. For $\rho\in[0,1]$, these types of dropouts are defined as:
\begin{itemize}
  \item Type I (Drp$^\text{(I)}$): randomly drop out each unit including input unit (corresponding to input feature) with probability $1-\rho$.
  \item Type II (Drp$^\text{(II)}$): randomly drop out each weight with probability $1-\rho$.
  \item Type III (Drp$^\text{(III)}$): randomly drop out each weight and unit including input unit (corresponding to input feature)  with probability $1-\rho$.
\end{itemize}

We assume that a full-connected neural network has $k$ hidden layers, and the $i$th hidden layer has $m_i$ hidden units. The general output for this neural network is given by
\[
f(\w,\x)=\langle\w_1^{[k]},\Psi_k\rangle \text{ with }\Psi_{i}=(\sigma(\langle\w^{[i-1]}_1,\Psi_{i-1}\rangle), \ldots, \sigma(\langle\w^{[i-1]}_{m_i},\Psi_{i-1}\rangle))\text{ for }i\in[k]
\]
and $\Psi_0=\x$, where $\w=(\w_1^{[k]},\w^{[k-1]}_1,$ $\ldots,\w^{[k-1]}_{m_{k}}, \ldots, \w^{[0]}_1,\ldots,\w^{[0]}_{m_1})$ in which each $\w^{[j]}_i$ has the same size as $\Psi_j$ and $\sigma$ is an activation function.

Throughout this work, we assume that activation function $\sigma$ is Lipschitz with constant $L$ and $\sigma(0)=0$, and many commonly used activation functions satisfy such assumptions, e.g., \textit{tanh}, \textit{center sigmoid}, \textit{relu} \cite{Nair:Hinton2010}, etc.

Formally, three types of dropouts for the full-connected network are defined as:
\begin{itemize}
\item The output for Drp$^\text{(I)}$ (first type) is given by
\begin{equation}\label{eq:t1}
\begin{split}
&f^\text{(I)}(\w,\x,\r)=\langle\w_1^{[k]},\Psi_k\odot \r^{[k]}\rangle \text{ with } \\
&\Psi_{i}=(\sigma(\langle\w^{[i-1]}_1,\Psi_{i-1}\odot \r^{[i-1]}\rangle),\ldots,\sigma(\langle\w^{[i-1]}_{m_i},\Psi_{i-1}\odot \r^{[i-1]}\rangle))
\end{split}
\end{equation}
for $i\in[k]$ and $\Psi_0=\x$. Here $\r=(\r^{[k]},\ldots,\r^{[1]}, \r^{[0]})$, each $\r^{[i]}$ has the same size with $\Psi_i$ and each element in $\r^{[i]}$ is drawn i.i.d. from Bern$(\rho)$.

\item
The output for Drp$^\text{(II)}$ (second type) is given by
\begin{equation}\label{eq:t2}
\begin{split}
&f^\text{(II)}(\w,\x,\r)=\langle\w_1^{[k]}\odot \r_1^{[k]},\Psi_k\rangle \text{ with }\\
&\Psi_{i}=(\sigma(\langle\w^{[i-1]}_1\odot \r^{[i-1]}_1, \Psi_{i-1}\rangle), \ldots,\sigma(\langle\w^{[i-1]}_{m_i} \odot \r_{m_i}^{[i-1]},\Psi_{i-1}\rangle))
\end{split}
\end{equation}
for $i\in[k]$ and $\Psi_0=\x$. Here $\r=\{\r_1^{[k]},\r^{[k-1]}_1,\ldots,$ $\r^{[k-1]}_{m_{k}}, \ldots, \r^{[0]}_1,\ldots,\r^{[0]}_{m_1}\}$, and for $0\leq j\leq k$, $\r^{[j]}_{i}$ has the same size with $\Psi_j$, and each element in $\r_{i}^{[j]}$ is drawn i.i.d. from Bern($\rho$).

\item The output for Drp$^\text{(III)}$ (third type) is given by
\begin{equation}\label{eq:t3}
\begin{split}
&f^\text{(III)}(\w,\x,\r)=\langle\w_1^{[k]}\odot\r^{[k]}_1,\Psi_k\odot\r^{[k]}_2\rangle\text{ with}\\
&\Psi_{i}=(\sigma(\langle\w^{[i-1]}_1\odot \r^{[i-1]}_1, \Psi_{i-1}\odot\r^{[i-1]}_{m_i+1}\rangle),\ldots,\sigma (\langle \w^{[i-1]}_{m_i} \odot \r^{[i-1]}_{m_i},\Psi_{i-1} \odot\r^{[i-1]}_{m_i+1} \rangle))
\end{split}
\end{equation}
for $i\in[k]$ and $\Psi_0=\x$. Here $\r=(\r_1^{[k]},\r_2^{[k]},\ldots, $ $\r^{[0]}_1,\ldots, \r^{[0]}_{m_1+1})$, and for $0\leq j\leq k$, $\r^{[i]}_{j}$ has the same size with $\Psi_j$, and each element in $\r^{[i]}_{j}$ is drawn i.i.d. from Bern$(\rho)$.
\end{itemize}

Given a set $\W$, we denote by
\begin{eqnarray}
\F^\text{(I)}_\W&=& \{f^\text{(I)}(\w,\x,\r)\colon \w\in\W\}\label{eq:t4}\\
\F^\text{(II)}_\W&=& \{f^\text{(II)}(\w,\x,\r)\colon \w\in\W\}\label{eq:t5}\\
\F^\text{(III)}_\W&=& \{f^\text{(III)}(\w,\x,\r)\colon \w\in\W\}\label{eq:t6}
\end{eqnarray}
where $f^\text{(I)}(\w,\x,\r)$, $f^\text{(II)}(\w,\x,\r)$ and $f^\text{(III)}(\w,\x,\r)$ are defined in Eqs.~\ref{eq:t1}-\ref{eq:t3}.

We will focus on full-connected neural networks, either shallow ones (with none or one hidden layer) and deep ones (with more hidden layers ).

\subsection{Shallow Network without Hidden Layer}
We first consider the shallow network without hidden layer, and therefore, the output is a linear function, i.e., $f(\w,\x)=\langle\w,\x\rangle$. Further, the outputs for Drp$^\text{(I)}$, Drp$^\text{(II)}$ and Drp$^\text{(III)}$ are given, respectively, by
\begin{eqnarray*}
f^\text{(I)}(\w,\x,\r)&=& \langle\w,\x\odot\r\rangle\\
f^\text{(II)}(\w,\x,\r)&=& \langle\w\odot\r,\x\rangle\\
f^\text{(III)}(\w,\x,(\r_1,\r_2))&=&\langle\w\odot\r_1,\x\odot\r_2\rangle
\end{eqnarray*}
where $\r$, $\r_1$ and $\r_2$ are of size $d$, and each element in $\r$, $\r_1$ and $\r_2$ is drawn i.i.d. from Bern$(\rho)$. The following theorem shows the Rademecher complexity for three types dropouts:
\begin{thm}\label{thm:linear}
Let $\mathcal{W}=\{\w\colon \|\w\|<B_1\}$, $\mathcal{X}=\{\x\colon \|\x\|\leq B_2\}$, and $\F^\text{(I)}_\W$, $\F^\text{(II)}_\W$ and $\F^\text{(III)}_\W$ are defined in Eqs.~\ref{eq:t4}-\ref{eq:t6}. Then, we have
\[
\mathfrak{R}_n(\F^{(1)}_\W)=\mathfrak{R}_n(\F^{(2)}_\W)\leq B_1B_2\sqrt{\rho/n} \quad\text{ and }\quad\mathfrak{R}_n(\F^{(3)}_\W)\leq B_1B_2\rho/\sqrt{n}.
\]
\end{thm}
If we do not drop out any weights and input units (corresponding to input features), i.e., $\rho=1$, then the above theorem gives a similar estimation for the Rademacher complexity of linear function space as stated in \cite[Theorem 3]{Kakade:Sridharan:Tewari2008}. Also, these complexities are independent to feature dimension, and thus can be applied to high-dimensional data. In addition, such result has independent interests in missing feature problems.

\begin{proof}
From $\langle\w\odot\r,\x\rangle=\langle\w,\x\odot\r\rangle$, it is easy to prove $\mathfrak{R}_n(\F^\text{(I)}_\W)=\mathfrak{R}_n(\F^\text{(II)}_\W)$. For $S_n=\{\x_1,\ldots,\x_n\}$ and $RS_n=\{\r_1,\ldots,\r_n\}$, we have
\[
\hat{\mathfrak{R}}_n(\F^\text{(I)}_\W,S_n,RS_n) = \frac{1}{n} E_{\mathbf{\epsilon}}\sup_{\w\in\W} \sum_{i=1}^n\epsilon_i \langle\w,\x_i\odot\r_i\rangle
\]
where $\mathbf{\epsilon}=(\epsilon_1,\ldots,\epsilon_n)$ are rademacher variables, and this yields that
\[
\hat{\mathfrak{R}}_n(\F^\text{(I)}_\W,S_n,RS_n)=\frac{1}{n} E_\epsilon \sup_{\w\in\W} \Big\langle\w, \sum_{i=1}^n\epsilon_i\x_i\odot\r_i\Big\rangle.
\]
By using the Cauchy-Schwartz inequality $\langle a,b\rangle\leq \|a\|\|b\|$ and $\|\w\|\leq B_1$, we have
\begin{eqnarray*}
\lefteqn{\hat{\mathfrak{R}}_n(\F^\text{(I)}_\W,S_n,RS_n)\leq \frac{B_1}{n} E_{\mathbf{\epsilon}} \Big\|\sum_{i=1}^n\epsilon_i\x_i\odot\r_i\Big\|}\\
&=&\frac{B_1}{n}E_{\mathbf{\epsilon}}\Big(\sum_{i=1}^n\sum_{j=1}^n \epsilon_{i}\epsilon_{j} \langle \x_i\odot\r_i,\x_j\odot\r_j \rangle\Big)^{1/2} \\
&\leq& \frac{B_1}{n}\Big(\sum_{i=1}^n\sum_{j=1}^n E_{\epsilon_i,\epsilon_j} \epsilon_i\epsilon_j \langle \x_i\odot\r_i,\x_j\odot\r_j \rangle\Big)^{1/2}
\end{eqnarray*}
where the last inequality holds from Jensen's inequality. Since $E_{\epsilon_i,\epsilon_j} \epsilon_i\epsilon_j=0$ for $i\neq j$ and $E_{\epsilon_i} \epsilon_i\epsilon_i=1$ for rademacher variables, we have
\begin{equation}\label{eq:sing:t1}
\hat{\mathfrak{R}}_n(\F^\text{(I)}_\W,S_n,RS_n)\leq \frac{B_1}{n}\Big(\sum_{i=1}^n\langle \x_i\odot\r_i,\x_i\odot\r_i \rangle\Big)^{\frac{1}{2}}.
\end{equation}
Based on the above inequality, it holds that
\begin{eqnarray*}
\lefteqn{\mathfrak{R}_n(\F^\text{(I)}_\W) =E_{S_n,RS_n}[\hat{\mathfrak{R}}_n(\F^\text{(I)}_\W,S_n,RS_n)]}\\
&\leq&\frac{B_1}{n}E_{S_n,RS_n}\Big(\sum_{i=1}^n\langle \x_i\odot\r_i,\x_i\odot\r_i \rangle \Big)^{1/2}\\
&\leq&\frac{B_1}{n}E_{S_n}\Big(\sum_{i=1}^nE_{\r_i} \langle \x_i\odot\r_i,\x_i\odot\r_i \rangle \Big)^{1/2}
\end{eqnarray*}
where the second inequality holds from Jensen's inequality. Finally, we have
\[
\mathfrak{R}_n(\F^\text{(I)}_\W)\leq B_1B_2\sqrt{\rho/n}
\]
from $\|\x_i\|\leq B_2$ and Eq.~\ref{eq:ulem2}.

In a similar manner, we have $\mathfrak{R}_n(\F^\text{(III)}_\W)\leq B_1B_2\rho/\sqrt{n}$ from $\langle\w\odot\r_1,\x\odot \r_2\rangle =\langle\w,\x\odot\r_1 \odot\r_2\rangle$ and Eq.~\ref{eq:ulem3}. This theorem follows as desired.
\end{proof}

\subsection{Shallow Network with One Hidden Layer}
We consider the shallow network with only one hidden layer, and assume that the hidden layer has $m$ hidden units. The output for such network is given by
\[
f(\w,\x)=\langle\w^{[1]},\Psi(\w^{[0]}_1,\ldots,\w^{[0]}_m,\x)\rangle
\]
with
\begin{equation}\label{eq:temp1}
\Psi(\w^{[0]}_1,\ldots,\w^{[0]}_m,\x)=(\sigma(\langle\w^{[0]}_1, \x\rangle),\ldots,\sigma(\langle\w^{[0]}_m, \x\rangle))
\end{equation}
where $\w=(\w^{[1]},\w^{[0]}_1,\ldots, \w^{[0]}_m)$, and $\w^{[1]}$ and $\w^{[0]}_i$ ($i\in[m]$) are of size $m$ and $d$, respectively.

From Eqs.~\ref{eq:t1}-\ref{eq:t3}, the outputs for Drp$^\text{(I)}$, Drp$^\text{(II)}$ and Drp$^\text{(III)}$ are given, respectively, by
\[
f^\text{(I)}(\w,\x,\r)=\langle\w^{[1]}, \r_1^{[1]}\odot \Psi(\w^{[0]}_1,\ldots,\w^{[0]}_m,\x \odot \r_1^{[0]}) \rangle
\]
\[
f^\text{(II)}(\w,\x,\r)=\langle\w^{[1]}\odot \r_1^{[1]}, \Psi(\w^{[0]}_1\odot\r^{[0]}_1,\ldots,\w^{[0]}_m\odot\r^{[0]}_m,\x)\rangle
\]
and
\[
f^\text{(III)}(\w,\x,\r)=\langle\w^{[1]}\odot \r_1^{[1]},\r_2^{[1]} \odot\Psi(\w^{[0]}_1\odot\r^{[0]}_1,\ldots,\w^{[0]}_m\odot\r^{[0]}_1, \x \odot \r^{[0]}_{m+1}) \rangle
\]
where $\Psi$ is defined in Eq.~\ref{eq:temp1}. Here $\r^{[1]}_i$ and $\r^{[0]}_j$ are of size $m$ and $d$, respectively, and each element in $\r^{[1]}_i$ and $\r^{[0]}_j$ is drawn i.i.d. from Bern$(\rho)$. The following theorem shows the Rademecher complexity for three types dropouts.
\begin{thm}\label{thm:twolayer}
Let $\W=\{(\w^{[1]},\w^{[0]}_1,\ldots, \w^{[0]}_m)\colon \|\w^{[1]}\|_1\leq B_1,\|\w^{[0]}_i\|\leq B_0 \}$, $\mathcal{X}=\{\x\in\mathbb{R}^d\colon \|\x\|\leq \hat{B}\}$ and $\F^\text{(I)}_\W$, $\F^\text{(II)}_\W$, $\F^\text{(III)}_\W$ are defined in Eqs.~\ref{eq:t4}-\ref{eq:t6}. Suppose that the activation $\sigma$ is Lipschitz with constant $L$ and $\sigma(0)=0$. Then, we have
\[
\mathfrak{R}_n(\F^\text{(I)}_\W)\leq\mathfrak{R}_n(\F^\text{(II)}_\W)\leq LB_1B_0\hat{B}\rho/\sqrt{n}\quad\text{ and }\quad\mathfrak{R}_n(\F^\text{(III)}_\W)\leq LB_1B_0\hat{B}\rho^2/\sqrt{n}.
\]
\end{thm}

\begin{proof}
We first have
\[
\mathfrak{R}_n(\F^\text{(I)}_\W) \leq \mathfrak{R}_n(\F^\text{(II)}_\W)
\]
from $f^\text{(II)}(\w,\x,\r)=f^\text{(I)}(\w,\x,\r')$ by selecting $\r^{[0]}_1=\cdots=\r^{[0]}_m=\r'^{[0]}$ and $\r^{[1]}=\r'^{[1]}$. In the following, we will estimate $\mathfrak{R}_n(\F^\text{(II)}_\W)$.

Given $S_n=\{\x_1,\ldots,\x_n\}$ and $RS_n=\{\r_1,\ldots,\r_n\}$, it holds that
\begin{eqnarray*}
\hat{\mathfrak{R}}_n(\F^\text{(II)}_\W,S_n,RS_n)&=&\frac{1}{n}E_{\epsilon}\big[\sup_{\w} \big\langle \w^{[1]}, \sum_{i=1}^n \epsilon_i\r_i^{[1]}\odot\Delta_i\big\rangle\big]\\
&\leq&B_1 E_{\epsilon}\big[\sup_{\w} \big\langle \frac{\w^{[1]}}{\|\w^{[1]}\|_1}, \frac{1}{n}\sum_{i=1}^n \epsilon_i\r_i^{[1]}\odot\Delta_i\big\rangle\big]
\end{eqnarray*}
where $\Delta_i=\Psi(\w^{[0]}_1\odot\r^{[0]}_{i1},\ldots, \w^{[0]}_m\odot\r^{[0]}_{im},\x_i)$ and $\Psi$ is defined by Eq.~\ref{eq:temp1} and the inequality holds from $\|\w^{[1]}\|_1\leq B$. From Lemma~\ref{lem:convex}, we have
\begin{equation}\label{eq:twolayer:temp1}
\hat{\mathfrak{R}}_n(\F^\text{(II)}_\W,S_n,RS_n) \leq \frac{B_1}{n} E_{\epsilon} \sup_{\w^{[0]}_1} \sum_{i=1}^n\epsilon_i r_{i1}^{[1]} \sigma(\langle\w^{[0]}_1\odot \r_{i1}^{[0]}, \x_i \rangle).
\end{equation}
From $\sigma(0)=0$ and $r_{i1}^{[1]}\in\{0,1\}$, we have $r_{i1}^{[1]} \sigma(t)=\sigma(r_{i1}^{[1]} t)$. Since $\sigma(\cdot)$ is Lipschitz with constant $L$, Lemma~\ref{lem:concentration} gives
\[
E_{\epsilon}\sup_{\w^{[0]}_1}  \sum_{i=1}^n\epsilon_i r_{i1}^{[1]} \sigma(\langle\w^{[0]}_1\odot \r_{i1}^{[0]}, \x_i \rangle)\leq L E_{\epsilon}\sup_{\w^{[0]}_1} \sum_{i=1}^n\epsilon_i  \langle\w^{[0]}_1\odot\r_{ij}^{[0]},  r_{i1}^{[1]}\x_i \rangle
\]
Similarly to the proof of Eq.~\ref{eq:sing:t1}, we have
\[
E_{\epsilon}\sup_{\w^{[0]}_1}  \sum_{i=1}^n\epsilon_i  \langle\w^{[0]}_1,  r_{i1}^{[1]}\x_i\odot\r_{i1}^{[0]} \rangle =B_0\Big(\sum_{i=1}^n \langle r_{i1}^{[1]} \x_i \odot \r_{i1}^{[0]}, r_{i1}^{[1]} \x_i \odot \r_{i1}^{[0]} \rangle\Big)^\frac{1}{2}.
\]
Combining with the previous analysis, we have
\begin{eqnarray*}
\lefteqn{\mathfrak{R}_n(\F^\text{(II)}_\W) =E[\hat{\mathfrak{R}}_n (\F^\text{(II)}_\W, S_n, RS_n)]} \\ &\leq&\frac{LB_1B_0}{n}E_{S_n,RS_n}\Big(\sum_{i=1}^n \langle r_{i1}^{[1]} \x_i \odot \r_{i1}^{[0]}, r_{i1}^{[1]} \x_i \odot \r_{i1}^{[0]} \rangle\Big)^\frac{1}{2}\\
&\leq&\frac{LB_1B_0}{n}E_{S_n}\sum_{i=1}^nE_{RS_n}\langle r_{i1}^{[1]} \x_i \odot \r_i^{[0]}, r_{i1}^{[1]} \x_i \odot \r_i^{[0]} \rangle\\
&\leq& LB_1B_0\hat{B}\rho/\sqrt{n},
\end{eqnarray*}
where the last inequality holds from $\|\x_i\|\leq \hat{B}$ and Eq.~\ref{eq:ulem4}.

In a similar way, we can prove $\mathfrak{R}_n(\F^\text{(III)}_\W)\leq B_1B_0\hat{B}\rho^2/\sqrt{n}$ by combining with Eqs.~\ref{eq:ulem4}-\ref{eq:ulem5}, and $\langle\w\odot\r_1,\x\odot\r_2\rangle=\langle \w,\x \odot\r_1\odot\r_2\rangle$. This completes the proof.
\end{proof}

\subsection{Deep Network with $k$ Hidden Layers}
Now we consider the neural network with $k$ ($k\geq1$) hidden layers, and the $i$th layer has $m_i$ hidden units ($i\in[k]$). The output for this neural network is given by
\begin{eqnarray*}
&f(\w,\x)=\langle\w_1^{[k]},\Psi_k\rangle \text{ with }\Psi_0=\x, \text{ and for }i\in[k]&\\
&\Psi_{i}=(\sigma(\langle\w^{[i-1]}_1,\Psi_{i-1}\rangle), \ldots, \sigma(\langle\w^{[i-1]}_{m_i},\Psi_{i-1}\rangle)),&
\end{eqnarray*}
and three types of dropout Drp$^\text{(I)}$, Drp$^\text{(II)}$ and Drp$^\text{(III)}$ are defined by Eqs.~\ref{eq:t1}-\ref{eq:t3}. The following theorem shows the Rademecher complexity for three types dropouts:
\begin{thm}\label{thm:twolayer}
Let $\W=\{(\w_1^{[k]}, \w^{[k-1]}_1, \ldots, \w^{[k-1]}_{m_{k}}, \ldots, $ $ \w^{[0]}_1, \ldots, \w^{[0]}_{m_2})\colon \|\w^{[0]}_i\|\leq B_0, \|\w^{[j]}_i\|_1\leq B_j \text{ for }j\geq1\}$, $\mathcal{X}=\{\x\in\mathbb{R}^d\colon \|\x\|\leq \hat{B}\}$, and $\F^\text{(I)}_\W$, $\F^\text{(II)}_\W$, and $\F^\text{(III)}_\W$ are defined in Eqs.~\ref{eq:t4}-\ref{eq:t6}. Suppose that the activation $\sigma$ is Lipschitz with constant $L$ and $\sigma(0)=0$. Then, we have
\[
\mathfrak{R}_n(\F^\text{(I)}_\W)\leq \mathfrak{R}_n(\F^\text{(II)}_\W)\leq \frac{\rho^{(k+1)/2}}{\sqrt{n}} L^k \hat{B}\prod_{j=0}^k B_j\quad \text{ and }\quad \mathfrak{R}_n(\F^\text{(III)}_\W)\leq \frac{\rho^{(k+1)}}{\sqrt{n}} L^k \hat{B}\prod_{j=0}^k B_j.
\]
\end{thm}
This theorem shows that dropout can lead to an exponential reduction of the Rademacher complexity with respect to the number of hidden layers within neural network. If we do not drop out any weights and units (including hidden units, or input units corresponding to input features), i.e., $\rho=1$, then the above theorem improves the results in \cite[Lemma 26]{Bartlett1998}. The Rademacher complexities are dependent on the norms of weights, but irrelevant to the number of units and weights in the network, as well as the dimension of input datasets.

\begin{proof}
We first have
\[
\mathfrak{R}_n(\F^\text{(I)}_\W)\leq \mathfrak{R}_n(\F^\text{(II)}_\W)
\]
from $f^\text{(II)}(\w,\x,\r)=f^\text{(I)}(\w,\x,\r')$ by selecting $\r^{[i]}_1=\cdots=\r^{[i]}_{m_{i+1}}=\r'^{[i]}$ for $0\leq i \leq k-1$ and $\r_1^{[k]}={\r'}_1^{[k]}$.

For any $S_n=\{\x_1,\ldots,\x_n\}$ and $RS_n=\{\r_1,\ldots,\r_n\}$, we will prove that
\begin{equation}\label{eq:deep-t2}
\hat{\mathfrak{R}}_n(\F^\text{(II)}_\W,S_n,RS_n) \leq \frac{L^k}{n} E_\epsilon\Big[ \sup_{\w_{1}^{[0]}} \sum_{i=1}^n \epsilon_i \big\langle\w_{1}^{[0]}\odot\r^{[0]}_{i,1}, \x_i \prod_{s=1}^k r^{[s]}_{i,j_{s+1},j_s} \big\rangle\Big] \prod_{j=1}^k B_j
\end{equation}
by induction on $k$, i.e., the number of layers in neural network, where $j_{k+1}=1$. It is easy to find the above holds for $k=1$ from Eq.~\ref{eq:twolayer:temp1}. Assume that Eq.~\ref{eq:deep-t2} holds for neural network within $k-1$ layers ($k\geq2$), and in the following we will prove for neural network with in $k$ layers.

For $\|\w_1^{[k]}\|_1\leq B_k$, we have
\begin{eqnarray*}
\lefteqn{\hat{\mathfrak{R}}_n(\F^\text{(II)}_\W,S_n,RS_n)=\frac{1}{n}E_{\epsilon}\sup_{\w} \sum_{i=1}^n\epsilon_i \langle \w_1^{[k]}\odot\r_{i1}^{[k]}, \Psi_{ik}\rangle}\\
&=&\frac{1}{n}E_{\epsilon}\sup_{\w} \Big\langle \w_1^{[k]}, \sum_{i=1}^n\epsilon_i\Psi_{ik}\odot\r_{i1}^{[k]}\Big\rangle
\leq\frac{B_k}{n}E_{\epsilon}\sup_{\w} \Big\langle \frac{\w_1^{[k]}}{\|\w_1^{[k]}\|_1}, \sum_{i=1}^n\epsilon_i\Psi_{ik}\odot\r_{i1}^{[k]}\Big\rangle
\end{eqnarray*}
where $\Psi_{ij}=(\sigma(\langle\w^{[j-1]}_1\odot \r_{i1}^{[j-1]},\Psi_{i,j-1}\rangle),\ldots,\sigma(\w^{[j-1]}_{m_j}$ $\odot \r_{m_j}^{[j-1]},\Psi_{i,j-1}))$ for $j\in[k]$ and $\Psi_{i0}=\x_i$. From Lemma~\ref{lem:convex}, we have
\[
E_{\epsilon}\sup_{\w} \Big\langle \frac{\w_1^{[k]}}{\|\w_1^{[k]}\|_1}, \sum_{i=1}^n\epsilon_i\Psi_{ik}\odot\r_{i1}^{[k]}\Big\rangle \leq E_{\epsilon}\sup_{\w}\sum_{i=1}^n\epsilon_i r_{i,1,1}^{[k]} \times\sigma(\langle\w^{[k-1]}_{1}\odot \r_{i,1}^{[k-1]},\Psi_{i,k-1}\rangle),
\]
which yields that
\[
\hat{\mathfrak{R}}_n(\F^\text{(II)}_\W,S_n,RS_n)\leq \frac{B_k}{n} E_{\epsilon}\sup_{\w}\sum_{i=1}^n\epsilon_i r_{i,1,1}^{[k]} \times\sigma(\langle\w^{[k-1]}_{1}\odot \r_{i,1}^{[k-1]},\Psi_{i,k-1}\rangle)
\]
Since $\sigma(0)=0$ and $\sigma$ is Lipschitz with constant $L$, Lemma~\ref{lem:concentration} gives
\begin{equation}\label{eq:temp2}
\hat{\mathfrak{R}}_n(\F^\text{(II)}_\W,S_n,RS_n)\leq \frac{LB_k}{n} E_{\epsilon}\sup_{\w} \sum_{i=1}^n\epsilon_i r_{i,1,1}^{[k]} \langle\w^{[k-1]}_{1}\odot \r^{[k-1]}_{i,1},\Psi_{i,k-1}\rangle.
\end{equation}
By using $r_{i,1,1}^{[k]}\sigma(t)=\sigma(r_{i,1,1}^{[k]}t)$, the term
\[
\frac{1}{n}E_{\epsilon}\sup_{\w} \sum_{i=1}^n\epsilon_i r_{i,1,1}^{[k]} \langle\w^{[k-1]}_{1}\odot \r^{[k-1]}_{i,1},\Psi_{i,k-1}\rangle
\]
can be viewed as the empirical Rademacher for another $k-1$ layers neural network with respect to sample $S'_n=\{\x_1r_{i,1,1}^{[k]},\ldots,\x_nr_{i,1,1}^{[k]}\}$ and $RS'_n=(\r^{[k-1]}_{1}, \r^{[k-2]}_1,\ldots,\r^{[k-2]}_{m_{k-2}},\r^{[0]}_{1},\ldots,\r^{[0]}_{m_{1}})$. Therefore, by our assumption that Eq.~\ref{eq:deep-t2} holds for any $k-1$ layers neural network, we have
\begin{eqnarray*}
\lefteqn{\frac{1}{n}E_{\epsilon}\sup_{\w} \sum_{i=1}^n \epsilon_i r_{i,1,1}^{[k]} \langle\w^{[k-1]}_{1}\odot \r^{[k-1]}_{i,1},\Psi_{i,k-1}\rangle}\\
&\leq&\frac{L^{k-1}\prod_{i=1}^{k-1}B_i}{n} E_{\epsilon} \sum_{i=1}^n\epsilon_i \big\langle\w_{1}^{[0]}\odot\r^{[0]}_{i,1}, \x_i \prod_{s=1}^{k-1} r^{[s]}_{i,j_{s+1},j_s} \big\rangle
\end{eqnarray*}
which proves that Eq.~\ref{eq:deep-t2} holds for $k$ layers neural network by combining with Eq.\ref{eq:temp2}.

Similarly to the proof of Theorem~\ref{thm:linear}, the term in Eq.~\ref{eq:deep-t2} can be further bounded by
\begin{eqnarray*}
\lefteqn{E_\epsilon\sup_{\w_{1}^{[0]}} \sum_{i=1}^n \epsilon_i \big\langle\w_{1}^{[0]}\odot\r^{[0]}_{i,1}, \x_i \prod_{s=1}^k r^{[s]}_{i,j_{s+1},j_s} \big\rangle}\\
&\leq& B_0\Big(\sum_{i}\langle\r^{[0]}_{i,1} \odot \x_i \prod_{s=1}^k r^{[s]}_{i,j_{s+1},j_s}, \r^{[0]}_{i,1} \odot \x_i \prod_{s=1}^k r^{[s]}_{i,j_{s+1},j_s}\rangle\Big)^{\frac{1}{2}}
\end{eqnarray*}
which yields that, from $\|\x_i\|\leq\hat{B}$ and Eq.~\ref{eq:ulem4},
\[
\mathfrak{R}_n(\F^\text{(II)}_\W)\leq \frac{1}{\sqrt{n}} L^{k}\rho^{(k+1)/2}\hat{B}\prod_{i=0}^k B_i.
\]

In a similar manner, we can prove that
\[
\mathfrak{R}_n(\F^\text{(III)}_\W)\leq \frac{1}{\sqrt{n}} L^{k}\rho^{k+1}\hat{B}\prod_{i=0}^k B_i,
\]
by using Eq.~\ref{eq:ulem5}, and this completes the proof.\qed
\end{proof}

\section{Conclusion}\label{sec:con}
Deep neural networks have witnessed great successes in various real applications. Many implementation techniques have been developed, however, theoretical understanding of many aspects of deep neural networks is far from clear. Dropout is an effective strategy to improve the performance as well as reduce the influence of overfitting during training of deep neural network, and it is motivated from the intuition of preventing complex co-adaptation of feature detectors. In this work, we study the Rademacher complexity of different types of dropout, and our theoretical results disclose that for shallow neural networks (with one or none hidden layer) dropout is able to reduce the Rademacher complexity in polynomial, whereas for deep neural networks it can amazingly lead to an exponential reduction of the Rademacher complexity. An interesting future work is to present tighter generalization bounds for dropouts. In this work, we focused on very fundamental types of dropouts. Analyzing other types of dropouts is another interesting issue for future work, and we believe that the current work sheds a light on the way for the analysis.

\bibliography{reference}\bibliographystyle{alpha}
\end{document}